\documentclass[lettersize,journal]{IEEEtran}
\usepackage{amsmath,amsfonts,amssymb}
\usepackage{amsthm}
\usepackage{algorithmic}
\usepackage{algorithm}
\usepackage{array}
\usepackage[caption=false,font=normalsize,labelfont=sf,textfont=sf]{subfig}
\usepackage{textcomp}
\usepackage{stfloats}
\usepackage{url}
\usepackage{verbatim}
\usepackage{xcolor,soul}
\usepackage{diagbox}
\usepackage{graphicx}
\usepackage{multirow}
\usepackage{cite}
\newtheorem{lemma}{Lemma}
\newtheorem{corollary}{Corollary}
\begin{document}

\title{Prompt-based Dynamic Token Pruning for Efficient Segmentation of Medical Images}

\author{Pallabi Dutta \IEEEmembership{Student Member, IEEE}, Anubhab Maity and 
Sushmita Mitra \IEEEmembership{Fellow, IEEE} 
\thanks{This work was supported by the J. C. Bose National Fellowship of Sushmita Mitra, grant no. JCB/2020/000033. (\textit{Corresponding author: Pallabi Dutta})}
\thanks{Pallabi Dutta, Anubhab Maity and Sushmita Mitra are with Machine Intelligence Unit, Indian Statistical Institute, Kolkata 700108, INDIA (e-mail: duttapallabi2907@gmail.com; maityanubhab.ds.ai.expert@gmail.com; sushmita@isical.ac.in)}}



\maketitle

\begin{abstract}
The high computational demands of Vision Transformers (ViTs) in processing a large number of tokens often constrain their practical application in analyzing medical images. This research proposes a Prompt-driven Adaptive Token ({\it PrATo}) pruning method to selectively reduce the processing of irrelevant tokens in the segmentation pipeline. The prompt-based spatial prior helps to rank the tokens according to their relevance. Tokens with low-relevance scores are down-weighted, ensuring that only the relevant ones are propagated for processing across subsequent stages. This data-driven pruning strategy improves segmentation accuracy and inference speed by allocating computational resources to essential regions. The proposed framework is integrated with several state-of-the-art models to facilitate the elimination of irrelevant tokens, thereby enhancing computational efficiency while preserving segmentation accuracy. The experimental results show a reduction of $\sim$ 35-55\% tokens; thus reducing the computational costs relative to baselines. Cost-effective medical image processing, using our framework, facilitates real-time diagnosis by expanding its applicability in resource-constrained environments.
\end{abstract}


\section{Introduction}

\IEEEPARstart{P}recise segmentation of anatomical or pathological areas in medical images is essential for diagnosis, treatment planning, surgical navigation, and personalized medicine. It enables clinicians to measure disease burden, assess treatment response, and improve patient outcomes. Deep learning \cite{lecun15} has revolutionized medical image analysis by offering powerful algorithms to automate complex tasks such as segmentation. Its ability to learn intricate patterns from large datasets without much human intervention has made it a popular choice. Vision Transformers (ViTs) \cite{dosovitskiy2020image} have recently gained prominence in the development of algorithms for computer vision tasks, including medical image segmentation \cite{dutta2025wavelet}.

ViTs employ a self-attention mechanism to capture long-range dependencies in images, thus facilitating global context awareness. This is beneficial for segmenting anatomical structures that exhibit various shapes and sizes. The global view enables ViTs to model interactions between far-apart image regions, which can be difficult for convolutional networks with their restricted receptive fields \cite{yu2023marrying}. Several ViT-based segmentation approaches {\it viz.}, UNETR \cite{hatamizadeh2022unetr}, TransUNET \cite{chen2021transunet}, Swin-UNETR \cite{hatamizadeh2021swin} and SegFormer \cite{xie2021segformer} have reported encouraging performance in medical image segmentation, achieving state-of-the-art results across various benchmarks.

Despite impressive performance, ViTs face limitations in resource-constrained environments due to the quadratic computational complexity of self-attention. This scales poorly with increasing input image dimensions. Consequently, the deployment of ViT-based segmentation models in real-time clinical environments becomes challenging, limiting their applicability in practice. 

Model compression techniques, such as pruning, are an effective way to mitigate the computational expenses associated with ViTs. Pruning reduces the computational cost of a model by removing less significant elements from its architecture \cite{li2024a}. Many tokens processed by ViTs are semantically redundant \cite{yu2022width}. Token pruning increases efficiency by reducing the number of processed tokens. 

Static pruning \cite{lopez2024filter,sun2025channel,adnan2024structured} removes the fixed components of the model architecture, before inference, irreversibly eliminating the components considered less significant. For example, TRAM \cite{marchetti2025efficient} modeled the self-attention mechanism through an MLP layer, with tokens corresponding to the input nodes of the MLP. The association between different tokens was represented by the weights of the MLP. The relevance score of each token was computed recursively on the MLP representation of the self-attention mechanism. The tokens with the lowest scores were discarded in the early layers to ensure that only relevant tokens were forwarded across the network. Intrahead pruning \cite{zhang2025intra} eliminated certain rows of the weight matrices to produce query, key, and value vectors for the self-attention mechanism. An importance score was calculated for each row to eliminate irrelevant entries. A supplementary relationship matrix was preserved to address dependency conflicts with subsequent layers. Shunted Self-Attention \cite{ren2022shunted} performed token reduction by merging across different attention heads to efficiently learn multiscale representations.

Although effective in minimizing model size and computational expense, static pruning is less flexible \cite{he2023structured}. Its rigid nature limits the ability to adjust to the unique features of each input image. This constraint poses significant challenges for the analysis of medical images, where the variability between images can be considerable. 

In contrast, dynamic pruning \cite{liu2024efficient} adaptively adjusts the computation inference by selectively deactivating model components based on input. Methods like DynamicViT \cite{rao2021dynamicvit},\cite{lin2023the} and \cite{salam2025skin} incorporated lightweight modules to predict the relevance scores of tokens at each stage of the network. The irrelevant tokens with respect to the input image were discarded. Evo-Vit \cite{xu2022evo} dynamically identified tokens, as informative or placeholder, based on a global class attention score. The placeholder tokens represented irrelevant ones, which were retained simply to maintain the grid structure of the attention map. They were aggregated into a single representative value which was then processed by the self-attention mechanism to reduce computations. 

Other works focused on architectural modifications of ViTs to reduce the total number of tokens. TA-ASF \cite{chen2024ta} implemented a two-stage token pruning method. An importance score was assigned to each token, followed by dividing the entire set of tokens into high- and low-importance sets. A subset of tokens was sampled from both groups to preserve low-importance tokens that might be of global importance in the later layers of the network. Subsequently, the tokens in the subset were merged on the basis of their similarity. The soft top-K token pruning \cite{zhou2023token} generated a score for each token with a lightweight module. Random noise was added to the scores to prevent the same tokens from being chosen repeatedly. The top K score tokens were subsequently selected in the forward pass. The algorithm next applied a differentiable function to generate probabilities of selecting a token, which were then used to update the score prediction network through backpropagation.

Early exit strategies \cite{tang2023dynamic,tian2025beyond} save computation by bypassing the processing of specific layers for a subset of tokens. However, they might introduce representation inconsistencies within the network. The expansion of the tokens \cite{huang2024general} improved the training time of ViT by gradually growing the total number of tokens from a small set of initial seed tokens. Each stage introduced new tokens different from the existing ones. The residual tokens were aggregated and processed along with the selected set.

 Although existing pruning techniques have significantly increased the efficiency of ViTs, there exist several drawbacks. Methods performing unstructured pruning hamper the grid-structure of the tokens; thereby, making them incompatible for hierarchical ViT models like Swin UNETR and UNETR. Hard token pruning strategies pose difficulty due to the loss of essential fine-grained information required to reconstruct the segmentation output. Incorporating auxiliary modules to calculate the token relevance scores increases the complexity and parameter count of the model. The importance scores for weighting tokens in dynamic token pruning approaches \cite{rao2021dynamicvit,lin2023the,salam2025skin} are often learned without explicit consideration of prior spatial knowledge about the varying target anatomical structures. This might lead to suboptimal segmentation results.

To address these limitations, we introduce a {\it Prompt-driven Adaptive Token ({\it PrATo})} pruning approach that explicitly incorporates structured spatial priors into the pruning process. In contrast to techniques that depend exclusively on auxiliary modules or heuristics for selecting tokens, our method utilizes the spatial information embedded in the box prompts to direct token retention. It facilitates input-specific adaptation of the network computation. This guarantees the preservation of tokens essential for depicting the structure of the target object, while redundant or less informative tokens are removed. By directly incorporating this spatial constraint, the proposed method improves the precision of the segmentation and optimizes computational efficiency. The research contributions are summarized as follows.
\begin{itemize}
    \item A novel prompt-driven dynamic token pruning framework to increase the efficiency of ViTs by integrating auxiliary spatial prompts. This helps to localize relevant regions related to the target structure.
    \item An efficient entropy-based scoring mechanism to quantify the relevance of different image tokens. This parameter-free approach prevents additional computational overhead during training and inference.
    \item The generalizability of the method is demonstrated across different state-of-the-art ViT-based medical image segmentation models on different publicly available datasets.
\end{itemize}

Section \ref{meth} provides a comprehensive overview of the steps involved in the {\it PrATo} token pruning framework.  The experimental results and details of implementation are presented in Section \ref{setup}, in the publicly available datasets ACDC \cite{bernard2018deep} and ISIC \cite{codella2018skin}. Qualitative and quantitative results derived from the application of our framework, as embedded in state-of-the-art ViT-based segmentation models, are described. Section \ref{concl} provides a summary of the concluding remarks.

\section{Methodology} \label{meth}

This section details our proposed {\it PrATo} framework for integration into ViT blocks of the prevalent ViT-based medical image segmentation models. Recent high-performance medical image segmentation models employ an {\it U-}shaped encoder-decoder architecture \cite{ronneberger2015u}. {\it PrATo} can be seamlessly integrated at different stages of segmentation models to adaptively remove irrelevant tokens from being processed in subsequent steps of the network. 

\subsection{Token generation in ViT}

ViTs transform an input image (or feature map volume) $I \in \mathbb{R}^{C \times H \times W}$ into a sequence of tokens \cite{dosovitskiy2020image}. Here, $H$, $W$ and $C$ represent the height, width, and channel dimensions of $I$. The input is partitioned into a set of non-overlapping patches $Z ; |Z| =\frac{HW}{p^2}$, with each patch having dimension $p^2$. They are flattened and linearly projected in $C'$-dimensional embedding space, resulting in the sequence $P' \in \mathbb{R}^{Z \times C'}$. These embedded vectors are called tokens, with $Z$ denoting the total number of tokens generated. Positional embedding is added to the tokens to retain their spatial locations. $P_0 \in \mathbb{R}^{Z \times C'}$ is the sequence of position-sensitive tokens that serve as input to a transformer encoder.

This encoder block consists of alternating layers of multihead self-attention units ($\mu$) and feed-forward network ($\Phi$) units. The $\mu$ captures global contextual relationships between different tokens, performing self-attention parallelly across multiple heads, to operate in various representational subspaces. $\Phi$ is independently applied to the output of each head, to enhance their representational capacity by including an additional non-linear transformation. The operations within a transformer block are expressed as:
\begin{equation}
    Y_{out} = \Phi[LN\{\mu(LN(P_0))\}] + P_0
\end{equation}

Here, $LN$ \cite{lei2016layer} is the Layer Normalization and $Y_{out} \in \mathbb{R}^{Z \times C'}$ is the final output of the transformer block.

\subsection{Prompt-driven adaptive pruning}

Prompts serve as a prior by offering high-level, task-specific information \cite{jia2022visual}. This directs the model to focus on segmenting specific structures within the input image, resulting in effective segmentation. Consequently, it allows token selection based on contextual information. This framework is referred to as {\it prompt-driven adaptive pruning}, where essential tokens are preserved, as shown in Fig. \ref{fig:prato}.

\begin{figure*}[t]
    \centering
    \includegraphics[width=\linewidth]{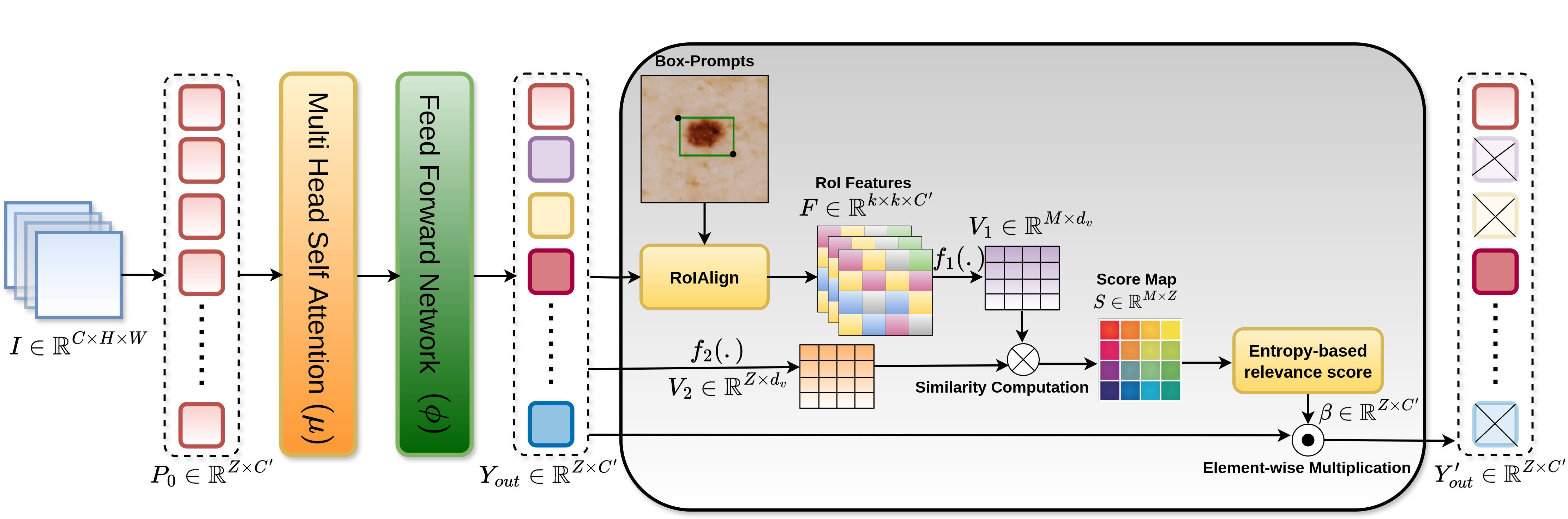}
    \caption{Illustration of the proposed {\it PrATo} framework for pruning ViT tokens.}
    \label{fig:prato}
\end{figure*}
\begin{figure*}[t]
    \centering
    \includegraphics[width=\linewidth]{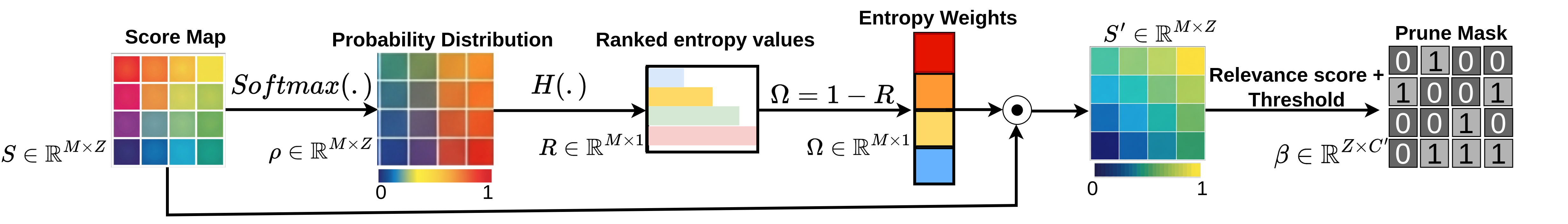}
    \caption{Illustration of the computation of pruning mask $\beta$ from the similarity score map $S$.}
    \label{fig:entropy}
\end{figure*}

 $Y_{out}$ is spatially rearranged to 2D feature maps of dimension $\mathbb{R}^{H'\times W'\times C'}$, where $H'\times W'=Z$. This reshaping maintains spatial coherence as the token order in $Y_{out}$ corresponds to the original spatial locations of the corresponding image patches. Subsequently, it is fed to RoIAlign \cite{he2017mask} along with box-prompts to extract region-specific features $F\in \mathbb{R}^{k \times k \times C'}$, as shown in Fig. \ref{fig:prato}. The box prompts serve as spatial priors, highlighting the most informative regions within the input. RoIAlign performs a soft crop and extracts the region-specific feature embeddings corresponding to the Region-of-Interest (RoI) defined in the prompt from $Y_{out}$. Each token in $Y_{out}$ is a contextual embedding of an image patch. Consequently, the ViT token grid will misalign with the box-prompt coordinates, making naive selection of tokens by hard-cropping infeasible. RoIAlign ensures a spatially coherent feature representation, unlike naive selection of tokens to preserve spatial consistency. 

$F$ is a summarization of the prompted region generated by RoIAlign. However, it is unable to identify globally relevant ViT tokens alone to segment the target structures. Therefore, a similarity score map $S \in \mathbb{R}^{M \times Z}$ is calculated between the region-specific feature, encoded by $F$ and $Y_{out}$, to measure the relevance of each token. The $F$ and $Y_{out}$ are first transformed into lower-dimensional embeddings $V_1 \in \mathbb{R}^{M \times d_V}$ and $V_2 \in \mathbb{R}^{Z \times d_V}$, respectively, with projection matrices $f_1 \in \mathbb{R}^{C' \times d_V}$ and $f_2 \in \mathbb{R}^{C' \times d_V}$. This standardizes the feature representations by mapping them from a specific architectural dimension $C'$ to a consistent low-dimensional embedding space $d_V$. $M$ denotes the number of feature vectors in $V_1$. The similarity score is computed as
\begin{equation}
    S(V_1, V_2) = V_1V_2^{T}/\sqrt{d_{V}}. 
\end{equation}

The next step involves computing the prune mask $\beta \in \mathbb{R}^{Z \times C'}$ from the similarity score mask $S$ to filter out irrelevant tokens, as shown in Fig. \ref{fig:entropy}. $S$ is normalized along the token dimension to generate a probability distribution $\rho \in \mathbb{R}^{M \times Z}$ by softmax operation. The probability $\rho_{ij}$ of the $i$ th feature vector of $V_1$ ($V_{1i}$) that it is relevant to the $j$th feature vector of $V_2$ ($V_{2j}$) is given as $\rho_{ij} =  \frac{e^{(S_{ij})}}{\sum_{k=1}^{Z} e^{S_{ik}}}$. $S_{ij}$ denotes the similarity score between $V_{1i}$ and $V_{2j}$. The softmax amplifies high similarities while suppressing lower ones, to distinctly recognize confident matches. Next, the Shannon entropy \cite{shannon1948mathematical} $H(.)$ is calculated for the probability distribution of each $V_{1i}$, over all ViT tokens.  This yields
\begin{equation}
    H(V_1{_i}) = -\sum_{j=1}^Z\rho_{ij}\log(\rho_{ij}),
\end{equation}
with $H(V_1{_i})$ quantifying the uncertainty of association between $V_{1i}$ and the set of all ViT tokens. A lower value of $H(V_1{_i})$ implies a strong association with a subset of tokens. In contrast, a higher value indicates that $V_{1i}$ is uncertain about its association with the tokens. In lieu of adopting thresholding to filter tokens with low similarity, entropy assesses token relevance in a statistically significant manner, facilitating a confidence-aware selection process. This makes the framework interpretable and applicable across various images. Tokens exhibiting high confidence (low entropy) are preserved, while those that are ambiguous (high entropy) are downweighted.

An inverse entropy weighting scheme is applied to prioritize tokens having lower uncertainty values. The entropy values are first ranked in ascending order. The ranks $R_i|i\in[1,M]$ are normalized to ensure their uniform spacing between 0 and 1. Inverse entropy weights $ \Omega_i = 1 - R_i$, with $\Omega = \{\Omega_i|i\in[1,M]\}$, ensure that higher entropy values (more uncertainty) are assigned to lower weights and vice versa.

\begin{lemma}[Uniformity of Weights]
If the normalized ranks $R_i$ are uniformly distributed on $[0,1]$, then the transformed weights $\Omega$
are also uniformly distributed over $[0,1]$. In other words, applying the transformation $\Omega_i = 1 - R_i$ preserves the uniformity of the distribution, merely reversing the order of values.  
\end{lemma}

\begin{proof}
Since $R_i \sim \text{Uniform}(0,1)$, its cumulative distribution function (CDF) is  
\[
F_{{R}}(x) = P(R_i \leq x) = x, \quad x \in [0,1].
\]  
For the transformed weights $\Omega_i = 1 - R_i$, the CDF becomes  
\[
F_{\Omega}(x) = P(\Omega_i \leq x) = P(1 - R_i \leq x) = P(R_i \geq 1 - x).
\]  
Since $R_i \sim \text{Uniform}(0,1)$, we substitute its CDF as  
\[
P(R_i \geq 1 - x) = 1 - P(R_i \leq 1 - x) = 1 - (1 - x) = x.
\]  
Thus,  
\[
F_{\Omega}(x) = x, \quad x \in [0,1].
\]  
Since this matches the CDF of a uniform distribution over $[0,1]$, we conclude that $\Omega_i \sim \text{Uniform}(0,1)$.  
\end{proof}

The weighted similarity score map becomes $\tilde{S} = S.\Omega$, the overall relevance $r_i$ of the $i$ th symbol being the mean of the weighted similarity scores. 
\begin{equation}
r_i = \frac{1}{C'}\sum_{j=1}^{C'} \tilde{S}_{ij}.
\end{equation}
The final step modifies the original tokens based on their relevance scores. A binary mask $\beta \in \{0,1\}^Z$ is generated by applying a threshold $\tau$ on the relevance scores as 
\begin{equation}
\beta_i =
\begin{cases} 
1, & \text{if } \sigma(r_i) > \tau, \\
0, & \quad \text{otherwise}.
\end{cases}
\end{equation}
Here, $\sigma$ is the sigmoid function to squash relevance scores in the range $[0,1]$, and $\beta$ is applied to ViT tokens $Y_{out}$ to effectively remove tokens that are deemed irrelevant based on our proposed prompt-guided weighting and thresholding. The masked token $Y_{out}'$ is generated by element-wise multiplication ($\odot$) of $\beta$ with $Y_{out}$, {\it i.e.} $Y_{out}' = \beta \odot Y_{out}.$ 

\begin{lemma}[Retaining high-relevance token set]
If $H(V_{1i})$ is low, there exists a typical set of high-relevance tokens $S_T$ which concentrates the effective information capacity. For the threshold value $\tau$, $S_T \subseteq S_P$, where $S_P$ is the final set of retained tokens.
\end{lemma}

\begin{proof}
    Since $H(V_{1i})$ is low, it satisfies $H(V_{1i})\leq \log_2(Z)-\Delta$, for some token reduction factor $\Delta>0$.\\
    By fundamental property of typical sets, \\
    $\exists$ $S_T\subset Z$ such that: $S_T\leq Z.2^{-\Delta}$ and $S_T$ concentrates the effective information capacity as follows:\\
    \begin{equation*}
        \sum_{k\in S_T}\rho_{ik}\geq 1-\epsilon; for \: \epsilon>0
    \end{equation*}
    $\therefore$ for any token $k\in S_T$ and $l\notin S_T$;
    \begin{equation} \label{eq:1}
        \rho_{ik}>>\rho_{il}
    \end{equation}
    The weights $\Omega$ is a monotonically decreasing function of $H(.)$ as for low $H(.)$ the corresponding $\Omega$ is high. Additionally, $\rho_{ik}$ is a monotonically increasing function of the similarity score $S_{ik}$. $\therefore$ from Eq. (\ref{eq:1}) $\implies S_{ik}>>S_{il}$.\\
    $\therefore$ for important tokens $i\in S_T$ and unimportant tokens $j\notin S_T$, we get:\\
    \begin{equation} \label{eq:2}
        r_i>>r_j
    \end{equation}
    Thus, the algorithm assigns higher scores to important tokens and lower scores to unimportant ones. Since, there exists a significant gap between the relevance scores, $\exists\:r_{min}=min_{i\in S_T}\{r_i\}$ and $r_{max}=max_{j\notin S_T}\{r_j\}$\\
    By Eq. (\ref{eq:2}),
    \begin{equation*}
        r_{min}>r_{max}
    \end{equation*}
    A token $k$ is retained if $\sigma(r_k)>\tau$. $\therefore$ $r_k>\sigma^{-1}(\tau)$. \\
    $\implies r_{min}>\sigma^{-1}(\tau)>r_{max}$\\
    $\forall i\in S_T$, the relevance score $r_i>\sigma^{-1}(\tau)$ which ensures every token in $S_T$ is retained.\\
    Thus, $S_T\subseteq S_P$.
    \end{proof}

\begin{corollary}
    Given the token reduction factor $\Delta >0$, then $|S_T|<Z$
\end{corollary}
\begin{proof}
    By the fundamental property of typical sets we know,\\
    \begin{equation} \label{eq:3}
        |S_T|\leq Z.2^{-\Delta}
    \end{equation}
    From the corollary we know,
    \begin{align}
    \Delta>0
        &\implies 2^{\Delta}>2^0 \\
        &\implies 2^{\Delta}>1 \\
        &\implies 2^{-\Delta}<1 \\
        &\implies Z.2^{-\Delta}<Z
    \end{align}
    $\therefore$ from Eq. (\ref{eq:3}) we get,
    \begin{equation}
        |S_T|\leq Z.2^{-\Delta}<Z \implies|S_T|<Z
    \end{equation}
    
\end{proof}

\section{Implementation details} \label{setup}

This section summarizes the public datasets used in this study, the implementation details and a qualitative and quantitative analysis of experimental results in terms of performance metrics. Ablations and comparative study (with related state-of-the-art literature), demonstrated the efficacy of our {\it PrATo} framework.

\subsection{Datasets}

The data set {\it Automated Cardiac Diagnosis Challenge (ACDC)} \cite{bernard2018deep}, acquired from the University Hospital of Dijon, France, contains 150 volumes of cardiac magnetic resonance images (MRI). Trained radiologists annotated the volumes for different cardiac components, {\it viz.} Right Ventricle (RV), Left Ventricle (LV), and Myocardium. The voxel intensities were clipped in the range [-170, 250] HU and subsequently normalized to the range of [0,1]. Data augmentation techniques, {\it viz.} random flipping, cropping and rotate were employed to expand the training dataset.

The data set {\it ISIC} \cite{codella2018skin} of the International Skin Imaging Collaboration is a collection of dermoscopic images for skin cancer research. Binary segmentation masks are created by experts where a pixel value of 0 indicates background while a pixel value of 1 denotes the lesion. The images were normalized, followed by the application of data augmentation techniques to enlarge the training data set.

\subsection{Performance Metrics and Loss function}

The {\it PrATo} framework is evaluated on state-of-the-art ViT-based segmentation models, {\it viz.} UNETR, TransUNET, SegFormer and Swin UNETR. The models were trained in Python 3.9, using PyTorch and MONAI libraries, on a NVIDIA RTX A5500 GPU with 24GB of memory. The adopted train-test-validation split is 80:10:10. The framework uses the tightest possible bounding box to encompass the entire target region.

A combined loss function, incorporating Dice loss ($\mathbf{\mathcal{L}_{d}}$) with Categorical Cross-Entropy loss ($\mathbf{\mathcal{L}_{c}}$) \cite{taghanaki2019combo} was used to train the models. The Dice loss function is commonly applied in image segmentation tasks that deal with class imbalance. Mathematically $\mathbf{\mathcal{L}_{d}}$ is expressed as
\begin{equation}
        \mathbf{\mathcal{L}_{d}} = \alpha - \sum_{\alpha=1}^N\left(\frac{2 \sum_{i=1}^V \hat{y}_{\alpha,i}y_{\alpha,i} + \epsilon}{\sum_{i=1}^V \hat{y}_{\alpha,i}+y_{\alpha,i} + \epsilon}\right),
\end{equation}
where $N$ represents the total number of classes, with $\hat{y}_{\alpha,i}$ and $y_{\alpha,i}$ denoting the predicted and true values, respectively, for the $i$th voxel concerning class $\alpha$. $V$ signifies the total number of voxels in the input and $\epsilon$ is the additive smoothing parameter to prevent division errors by zero.

The categorical cross-entropy loss $\mathbf{\mathcal{L}_{c}}$ measures the difference between the probability distributions of the predicted and the ground truth map. It is expressed as
\begin{equation}
    \mathcal{L}_{c}=-\frac{1}{V}\sum_{i=1}^V\sum_{\alpha=1}^C y_{\alpha,i}\log{(\hat{y}_{\alpha,i})}.
\end{equation}

The composite loss function, using both Dice and categorical cross-entropy loss, leverages the respective benefits of each component. While $\mathbf{\mathcal{L}_{d}}$ mitigates the issue of class imbalance between foreground and background pixels, the $\mathcal{L}_{c}$ component adjusts the trade-off between False Positives ($FP$) and False Negatives ($FN$) on the predicted output map \cite{taghanaki2019combo}. The composite loss function is 
\begin{equation}
    \mathbf{\mathcal{L}(\{\hat{\gamma},\gamma\};\Phi})=\mathbf{\mathcal{L}_{d}(\{\hat{\gamma},\gamma\},\Phi})+\mathbf{\mathcal{L}_{c}(\{\hat{\gamma},\gamma\},\Phi}),
\end{equation}
where $\hat{\gamma}$ and $\gamma$ denote the predicted and the ground truth map, respectively. $\Phi$ represents the model parameters.

The Dice Score Coefficient ($DSC$), Intersection-Over-Union ($IoU$) and 95\% Haursdorff distance ($HD95$) \cite{nguyen2023manet} were used to compare the segmentation performance
between the original baselines and their adaptations after incorporating {\it PrATo} framework. The metrics are mathematically represented as
\begin{equation}
    DSC=\frac{2TP}{2TP+FP+FN},
    \label{dsc}
\end{equation}
\begin{equation}
    IoU=\frac{TP}{TP+FP+FN},
    \label{iou}
\end{equation}
\begin{multline}
    HD95_{\alpha}(\hat{\gamma}, \gamma) = \max_{\alpha \in C}\{\max_{\hat{y}_{\alpha,i \in \hat{\gamma}}}\min_{y_{\alpha,j\in \gamma}}\delta(\hat{y}_{\alpha,i \in \hat{\gamma}},y_{\alpha,j\in \gamma}), \\
    \max_{y_{\alpha,i \in \gamma}}\min_{\hat{y}_{\alpha,j\in \hat{\gamma}}}\delta (\hat{y}_{\alpha,i \in \hat{\gamma}},y_{\alpha,j\in \gamma}) \},
\end{multline}
where $TP$ represents True Positives and $\delta(.)$ corresponds to the Euclidean distance.

\subsection{Ablation Study}

\begin{table}[t]
\caption{Component ablation study of {\it PrATo} on {\it ACDC} dataset. The best results are marked in bold.}
\label{table:abl1}
\centering
\resizebox{\linewidth}{!}{
\begin{tabular}{|c|ccc|c|c|}
\hline
\multirow{3}{*}{\textbf{PrATo Variants}} & \multicolumn{3}{c|}{\textbf{\textit{DSC}}}                                                                                  & \multirow{3}{*}{\textbf{\textit{mIoU}}} & \multirow{3}{*}{\textbf{\textit{mHD95}}} \\ \cline{2-4}
                                & \multicolumn{2}{c|}{\textbf{Ventricle}}                                              & \multirow{2}{*}{\textbf{Myocardium}} &                       &                        \\ \cline{2-3}
                                & \multicolumn{1}{c|}{\textbf{Left}}            & \multicolumn{1}{c|}{\textbf{Right}}           &                             &                       &                        \\ \hline
w/o Entropy                     & \multicolumn{1}{c|}{0.4923}          & \multicolumn{1}{c|}{0.6819}          & 0.7586                      & 0.5393                & 18.48                  \\ \hline
Naive Cropping                  & \multicolumn{1}{c|}{0.5272}          & \multicolumn{1}{c|}{0.7068}          & 0.7739                      & 0.5692                & \textbf{15.89}         \\ \hline
Proposed                        & \multicolumn{1}{c|}{\textbf{0.6146}} & \multicolumn{1}{c|}{\textbf{0.7171}} & \textbf{0.819}              & \textbf{0.6129}       & 19.87                  \\ \hline
\end{tabular}}
\end{table}

\begin{table}[t]
\caption{Ablation study of {\it PrATo} on {\it ACDC} dataset, over various threshold $\tau$.}
\label{table:abl2}
\centering
\begin{tabular}{|cc|c|c|l}
\cline{1-4}
\multicolumn{2}{|c|}{$\mathbf{\tau}$}                                & \textbf{GFLOPs} & \textbf{\textit{mDSC}}  &  \\ \cline{1-4}
\multicolumn{1}{|c|}{\multirow{4}{*}{Fixed}}       & 0.1  & 15.25  & 0.6131 &  \\
\multicolumn{1}{|c|}{}                            & 0.2  & 15.78  & 0.6425 &  \\
\multicolumn{1}{|c|}{}                            & 0.3  & 13.32  & 0.6701 &  \\
\multicolumn{1}{|c|}{}                            & 0.5  & 9.07   & 0.6612 &  \\ \cline{1-4}
\multicolumn{1}{|c|}{\multirow{3}{*}{Percentile}} & 25$^{th}$ & 26.78  & 0.6933 &  \\
\multicolumn{1}{|c|}{}                            & 50$^{th}$ & 37.31  & 0.6952 &  \\
\multicolumn{1}{|c|}{}                            & 75$^{th}$ & 18.66  & 0.6575 &  \\ \cline{1-4}
\end{tabular}
\end{table}

\begin{table}[t]
\caption{Ablation study for spatial dimension of the RoI Align ($k$) on {\it ACDC} dataset.}
\label{table:abl3}
\centering
\resizebox{0.8\linewidth}{!}{
\begin{tabular}{|c|c|c|c|c|}
\hline
$\mathbf{k }=$   & 3              & 5               & 7      & 9      \\ \hline
\textbf{\textit{mDSC}}  & 0.6761         & \textbf{0.7169} & 0.6405 & 0.6627 \\ \hline
\end{tabular}}
\end{table}

The contribution of the core components of {\it PrATo} is presented in Table \ref{table:abl1}. Two variants of {\it PrATo} are compared with the proposed framework, specifically, one without entropy-based weighting and the other replacing RoIAlign with Naive cropping. The second variant directly selects the tokens with center coordinates located within the box prompt. The proposed {\it PrATo} framework attains the highest $DSC$ value among all cardiac organs and demonstrates superior average $IoU$ relative to other variants. This indicates that the integration of RoIAlign for feature extraction and entropy-based token scoring effectively identifies relevant tokens. The variant lacking the entropy-weighting scheme demonstrates a notable decrease in $DSC$ and $mIoU$ values. This variant exclusively employs the unrefined features of RoIAlign to identify important tokens. Therefore, the entropy-weighting scheme is necessary to filter ambiguous features, leading to precise token selection. Naive cropping selects all tokens within the box indiscriminately, as it does not possess the advanced feature extraction capabilities of RoIAlign. This can be evidenced by the significant drop in $DSC$ values of the left ventricle and myocardium. Therefore, RoIAlign is the superior feature extraction approach for this task.

However, the average HD95 value ($mHD95$) is lower with the naive cropping approach than with the proposed framework. The fine-grained details especially around the edges are smoothed out by directly selecting all the underlying tokens within the box-prompt. This spatially averaged query directs the similarity mechanism to select contiguous groups of tokens, leading to predictions characterized by simpler boundaries. Despite the reduced penalties for smoother boundaries as indicated by the $HD95$ metric, the overall shape of the target structure is not accurately predicted, as demonstrated by the lower values of $DSC$ and $IoU$.

Table \ref{table:abl2} quantifies the experimental results of fixed threshold values (0.1, 0.2, 0.3, 0.5) and adaptive percentile-based threshold values ($25^{th}, 50^{th}, 75^{th}$) for $\tau$. Lower values of $\tau$, {\it viz.} 0.1 and 0.2 exhibits a reduced value of mean {\it DSC} over the different classes in {\it ACDC}. This suggests that allowing low-relevance token processing tends to introduce noise, which impacts the final output quality. Although increasing the threshold value to 0.5 leads to a significant decrease in GFLOPs, it also results in a slight drop in segmentation accuracy, indicating over-pruning.

In contrast, percentile-based thresholding has higher segmentation accuracy, as evidenced by higher {\it mDSC} values. Lower percentile values ($25^{th}$ and $50^{th}$) retain comparatively more tokens than higher percentiles ($75^{th}$), as indicated by the increase in Giga Floating Point Operations (GFLOPs). $\tau$ at the $25^{th}$ percentile is found to have the optimal balance between accuracy and computational efficiency. Discarding 75\% of the total tokens results in suboptimal segmentation performance due to information loss. Adaptive percentile-based thresholding provides greater flexibility to the token pruning mechanism than fixed threshold value and achieves a superior balance between segmentation performance and computational costs. The fixed threshold values remain constant for all inputs, making them inflexible to the distinct score distribution of each input. In contrast, percentile-based thresholding always preserves a constant proportion of relevant tokens in relation to the specific input. The risk of discarding essential tokens is minimal, thus maintaining the segmentation performance.

Table \ref{table:abl3} presents the results of the ablation study performed to analyze the impact of different spatial dimensions of the region-specific characteristics derived from RoIAlign. Segmentation performance improves significantly from $k=3$ to 5. The $5\times 5$ feature map obtained by setting $k=5$ results in a high-dimensional query that allows the model to learn discriminative and fine-grained details from the bounded region of the box prompt. However, increasing the value of $k$ might include specific textures or noise unique to the small target region. Therefore, while finding its association with the rest of the tokens, the model fails to generalize and overfits to the specific input prompt. Consequently, the segmentation performance drops to higher values of $k$. Therefore, $k=5$ was chosen as the final value for the proposed framework.

\section{Results and Discussion}

\begin{table*}[t]
\caption{Comparative analysis of {\it PrATo} with other pruning frameworks on {\it ACDC} and {\it ISIC} datasets. The best results are marked in \textbf{bold}.}
\label{table:quan_f}
\centering
\resizebox{\linewidth}{!}{
\begin{tabular}{|c|ccccc|ccc|c|}
\hline
                                       & \multicolumn{5}{c|}{\textit{\textbf{ACDC}}}                                                                                                                                                                                                          & \multicolumn{3}{c|}{}                                                                                                                                                  &                                                \\ \cline{2-6}
                                       & \multicolumn{3}{c|}{\textit{\textbf{DSC}}}                                                                                               & \multicolumn{1}{c|}{}                                         &                                           & \multicolumn{3}{c|}{\multirow{-2}{*}{\textit{\textbf{ISIC}}}}                                                                                                          &                                                \\ \cline{2-4} \cline{7-9}
                                       & \multicolumn{2}{c|}{\textbf{Ventricle}}                                     & \multicolumn{1}{c|}{}                                      & \multicolumn{1}{c|}{}                                         &                                           & \multicolumn{1}{c|}{}                                        & \multicolumn{1}{c|}{}                                        &                                          &                                                \\ \cline{2-3}
\multirow{-4}{*}{\textbf{Framework}}   & \multicolumn{1}{c|}{\textbf{Left}}   & \multicolumn{1}{c|}{\textbf{Right}}  & \multicolumn{1}{c|}{\multirow{-2}{*}{\textbf{Myocardium}}} & \multicolumn{1}{c|}{\multirow{-3}{*}{\textit{\textbf{mIoU}}}} & \multirow{-3}{*}{\textit{\textbf{mHD95}}} & \multicolumn{1}{c|}{\multirow{-2}{*}{\textit{\textbf{DSC}}}} & \multicolumn{1}{c|}{\multirow{-2}{*}{\textit{\textbf{IoU}}}} & \multirow{-2}{*}{\textit{\textbf{HD95}}} & \multirow{-4}{*}{\textbf{Inference time (ms)}} \\ \hline
DynamicViT     & \multicolumn{1}{c|}{0.5115}          & \multicolumn{1}{c|}{0.5898}          & \multicolumn{1}{c|}{0.7302}                                & \multicolumn{1}{c|}{0.4879}                                   & 24.03                                     & \multicolumn{1}{c|}{0.8493}                                  & \multicolumn{1}{c|}{0.7539}                                  & 21.59                                    & 10.11                                          \\ \hline
EvoViT         & \multicolumn{1}{c|}{0.5418}          & \multicolumn{1}{c|}{0.681}           & \multicolumn{1}{c|}{0.8085}                                & \multicolumn{1}{c|}{0.5561}                                   & 24.5                                      & \multicolumn{1}{c|}{0.8541}                                  & \multicolumn{1}{c|}{0.7566}                                  & \textbf{8.30}                            & 59.15                                          \\ \hline
DToP                                   & \multicolumn{1}{c|}{0.4463}          & \multicolumn{1}{c|}{0.5598}          & \multicolumn{1}{c|}{0.6845}                                & \multicolumn{1}{c|}{0.4447}                                   & 21.7                                      & \multicolumn{1}{c|}{0.8511}                                  & \multicolumn{1}{c|}{0.7449}                                  & 35.41                                    & 15.23                                          \\ \hline
Random Token Masking & \multicolumn{1}{c|}{0.5094}          & \multicolumn{1}{c|}{0.5927}          & \multicolumn{1}{c|}{0.7136}                                & \multicolumn{1}{c|}{0.4905}                                   & \textbf{18.85}                            & \multicolumn{1}{c|}{0.8402}                                  & \multicolumn{1}{c|}{0.7377}                                  & 13.83                                    & \textbf{2.66}                                  \\ \hline
STP            & \multicolumn{1}{c|}{0.3958}          & \multicolumn{1}{c|}{0.4224}          & \multicolumn{1}{c|}{0.5677}                                & \multicolumn{1}{c|}{0.3388}                                   & 21.72                                     & \multicolumn{1}{c|}{0.8544}                                  & \multicolumn{1}{c|}{0.7483}                                  & 10.91                                    & 13.21                                          \\ \hline
TRAM                                   & \multicolumn{1}{c|}{0.4867}          & \multicolumn{1}{c|}{0.5448}          & \multicolumn{1}{c|}{0.6844}                                & \multicolumn{1}{c|}{0.4562}                                   & 20.08                                     & \multicolumn{1}{c|}{0.6424}                                  & \multicolumn{1}{c|}{0.5514}                                  & 40.02                                    & 3.07                                           \\ \hline
{\it PrATo}                                  & \multicolumn{1}{c|}{\textbf{0.6146}} & \multicolumn{1}{c|}{\textbf{0.7171}} & \multicolumn{1}{c|}{\textbf{0.819}}                        & \multicolumn{1}{c|}{\textbf{0.6129}}                          & 19.87                                     & \multicolumn{1}{c|}{\textbf{0.8634}}                         & \multicolumn{1}{c|}{\textbf{0.7678}}                         & 17.34                                    & 3.55                                           \\ \hline
\end{tabular}}
\end{table*}

\begin{figure*}[t]
    \centering
    \includegraphics[width=0.7\textwidth]{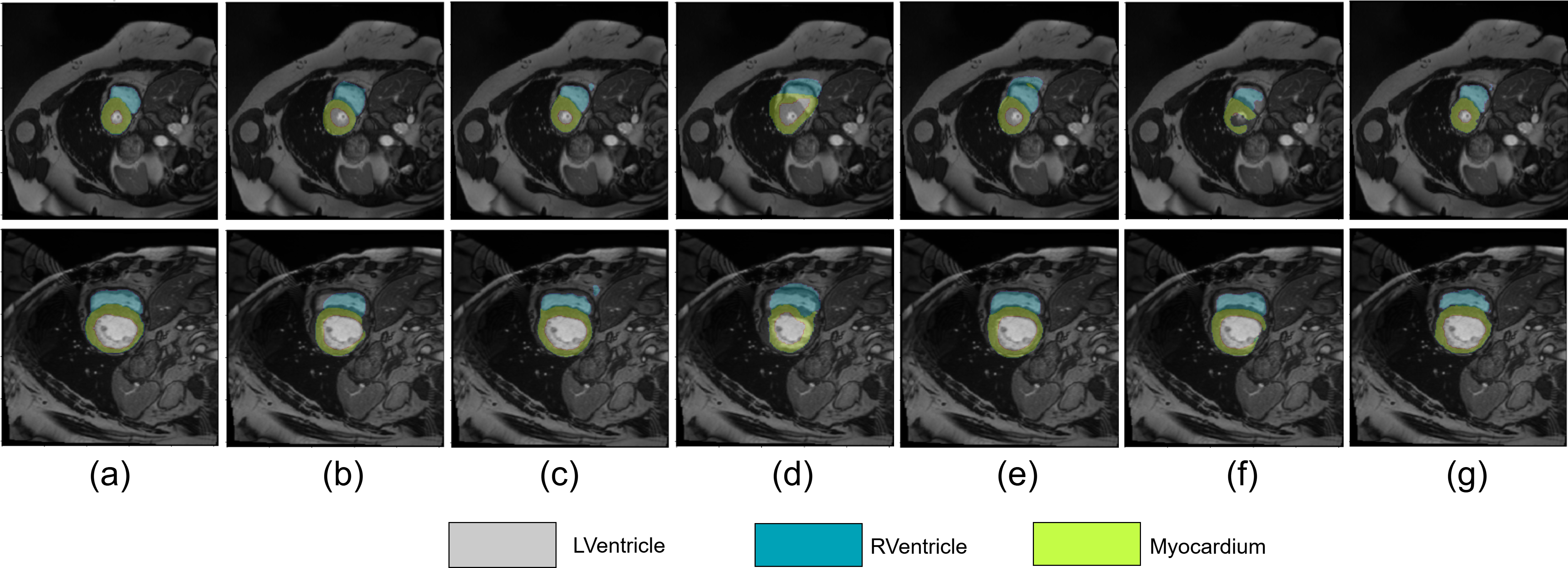}
    \caption{Sample segmentation maps for comparing {\it PrATo} with other pruning frameworks on the {\it ACDC} dataset. (a) Input MRI image with overlay ground truth and sample outputs from (b) Dynamic ViT, (c) EvoViT, (d) STP, (e) DToP, (f) Random Token Masking, and (g) {\it PrATo} frameworks.}
    \label{fig:qual_acdc}
\end{figure*}

\begin{figure*}[t]
    \centering
    \includegraphics[width=0.7\textwidth]{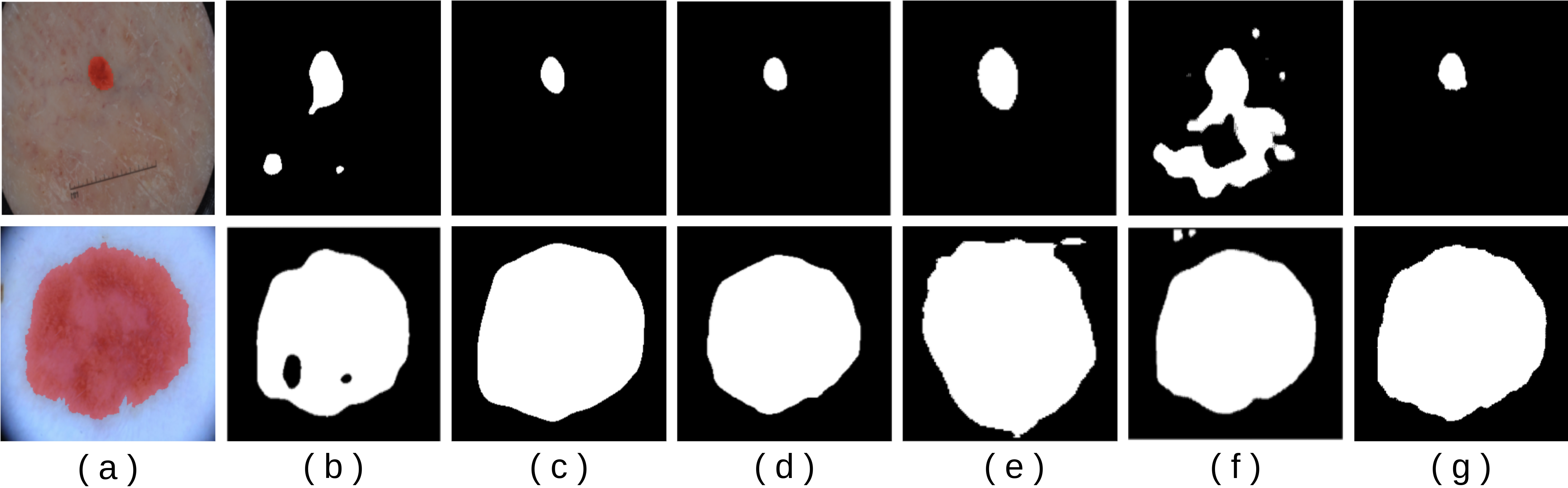}
    \caption{Sample segmentation maps for comparing {\it PrATo} with other pruning frameworks on the {\it ISIC} dataset. (a) Input dermoscopy image with overlay ground truth and sample outputs from (b) Random Token Masking, (c) Dynamic ViT, (d) EvoViT, (e) STP, (f) DToP, and (g) {\it PrATo} frameworks.}
    \label{fig:qual_isic}
\end{figure*}

Table \ref{table:quan_f} presents the experimental findings to understand the potential of the prompt-guided {\it PrATo} compared to other state-of-the-art pruning frameworks. The comparisons also include a crucial baseline, called Random Token Masking, which randomly drops a fixed number of tokens without any guidance from spatial priors like box-prompt. This highlights the impact of pruning without any auxiliary guidance, like spatial priors from box-prompts. Class-wise $DSC$ along with the average values $IoU$ and $HD95$ are reported for the {\it ACDC} dataset. {\it PrATo} shows superior performance in terms of $DSC$, achieving improvements of $7.28\%$, $3.61\%$ and $1.05\%$ for the left ventricle, right ventricle and myocardium, respectively, compared to the next best baseline. The mean $IoU$ exceeds that of the second-best method by nearly $7\%$. Random Token Masking demonstrates a marginal advantage over {\it PrATo} in terms of $mHD95$; however, its low $DSC$ and $mIoU$ values suggest under-segmentation of regions that align closely with the true organ regions, as depicted in Fig. \ref{fig:qual_acdc}(f).

A qualitative comparison is presented in Fig. \ref{fig:qual_acdc} to visually assess segmentation maps from different frameworks. The output of {\it PrATo} is visually the closest to the ground truth compared to other approaches. The overall shapes of the target structures are accurately captured, indicating the importance of spatial priors in retaining relevant tokens. STP produces imprecise output with inaccurate boundaries and under-segmented regions of the left ventricle and myocardium, respectively, as seen in Fig. \ref{fig:qual_acdc} (d). Dynamic ViT, EvoViT and DToP oversegments the left ventricle, as evident from the top row of Fig. \ref{fig:qual_acdc} (b), (c), and (e).

{\it PrATo} achieves the best performance based on $DSC$ and $IoU$ scores on the {\it ISIC} dataset, as shown in Table \ref{table:quan_f}. This indicates that the segmentation output of {\it PrATO} has the best overlap with the ground truth, accurately delineating the lesion region compared to other methods. However, EvoViT achieves the best $HD95$ score. This difference in boundary delineation efficacy between EvoViT and {\it PrATo} originates from their different pruning mechanisms. EvoViT retains irrelevant tokens to maintain the spatial grid structure, thus preserving finer boundary details. Although the box-prompt in {\it PrATo} serves as a robust prior to identifying the overall region, it represents a coarser form of guidance that may result in inaccuracies in boundary delineation.

Fig. \ref{fig:qual_isic} illustrates the sample segmentation maps of the {\it PrATo} and other pruning frameworks in the {\it ISIC} data set. The proposed framework generates visually precise segmentation output for both small and large lesions. The lesion boundaries are smooth and the overall shape aligns closely with the truth of the ground. Random Token Masking produces several under-segmented regions in the output. STP and DToP display numerous oversegmented areas characterized by irregular boundaries. 

{\it PrATo} outperforms several high-accuracy methods in inference speed, such as EvoViT and Dynamic ViT, as detailed in Table \ref{table:quan_f}. It is approximately 3x faster than Dynamic ViT and nearly 15x faster than the high-performing model EvoViT. The ability of {\it PrATo} to exhibit competitive performance with reduced inference times highlights the efficiency of the proposed prompt-guided mechanism. Although Random Token Masking and TRAM are comparatively faster than {\it PrATo}, their performance degradation makes them unsuitable for clinical use. 

\begin{table*}[t]
\caption{Quantitative analysis of the implementation of {\it PrATo} across various medical image segmentation baseline models on the {\it ACDC} and {\it ISIC} datasets. The DSC value is reported for each baseline and their corresponding pruned version.}
\label{table:quan}
\centering
\resizebox{0.7\linewidth}{!}{
\begin{tabular}{|c|cccccc|cc|}
\hline
\multirow{4}{*}{\diagbox{\textbf{Models}}{\textbf{Datasets}}} & \multicolumn{6}{c|}{\textbf{\textit{ACDC}}}                                                                                                                                                      & \multicolumn{2}{c|}{\textbf{\textit{ISIC}}}                                            \\ \cline{2-9} 
                        & \multicolumn{4}{c|}{\textbf{Ventricle}}                                                                                              & \multicolumn{2}{c|}{\multirow{2}{*}{\textbf{Myocardium}}} & \multicolumn{1}{c|}{\multirow{3}{*}{\textbf{Baseline}}} & \multirow{3}{*}{\textbf{Pruned}} \\ \cline{2-5}
                        & \multicolumn{2}{c|}{\
                        \textbf{Left}}                                & \multicolumn{2}{c|}{\textbf{Right}}                           & \multicolumn{2}{c|}{}                            & \multicolumn{1}{c|}{}                      &                         \\ \cline{2-7}
                        & \multicolumn{1}{c|}{\textbf{Baseline}}  & \multicolumn{1}{c|}{\textbf{Pruned}} & \multicolumn{1}{c|}{\textbf{Baseline}} & \multicolumn{1}{c|}{\textbf{Pruned}} & \multicolumn{1}{c|}{\textbf{Baseline}}         & \textbf{Pruned}       & \multicolumn{1}{c|}{}                      &                         \\ \hline
UNETR                   & \multicolumn{1}{c|}{0.5521} & \multicolumn{1}{c|}{0.5943}  & \multicolumn{1}{c|}{0.6952}         & \multicolumn{1}{c|}{0.7201}  & \multicolumn{1}{c|}{0.7773}        & 0.7820         & \multicolumn{1}{c|}{0.8330}                 & 0.8754                   \\ \hline
Swin UNETR              & \multicolumn{1}{c|}{0.6400} & \multicolumn{1}{c|}{0.6621}  & \multicolumn{1}{c|}{0.7940}         & \multicolumn{1}{c|}{0.7854}   & \multicolumn{1}{c|}{0.8462}        & 0.8543      & \multicolumn{1}{c|}{0.8850}                 & 0.8832                   \\ \hline
TransUNET               & \multicolumn{1}{c|}{0.6401} & \multicolumn{1}{c|}{0.7061}  & \multicolumn{1}{c|}{0.7693}         & \multicolumn{1}{c|}{0.8312 }  & \multicolumn{1}{c|}{0.8710}        & 0.8821     & \multicolumn{1}{c|}{0.8242}                 & 0.8253                 \\ \hline
SegFormer         & \multicolumn{1}{c|}{0.6841} & \multicolumn{1}{c|}{0.6791}  & \multicolumn{1}{c|}{0.8221}         & \multicolumn{1}{c|}{0.8163 }  & \multicolumn{1}{c|}{0.8981}        & 0.9024       & \multicolumn{1}{c|}{0.8121}                 & 0.8220                   \\ \hline
\end{tabular}}
\end{table*}

\begin{figure*}[t]
    \centering
    \includegraphics[width=0.9\linewidth]{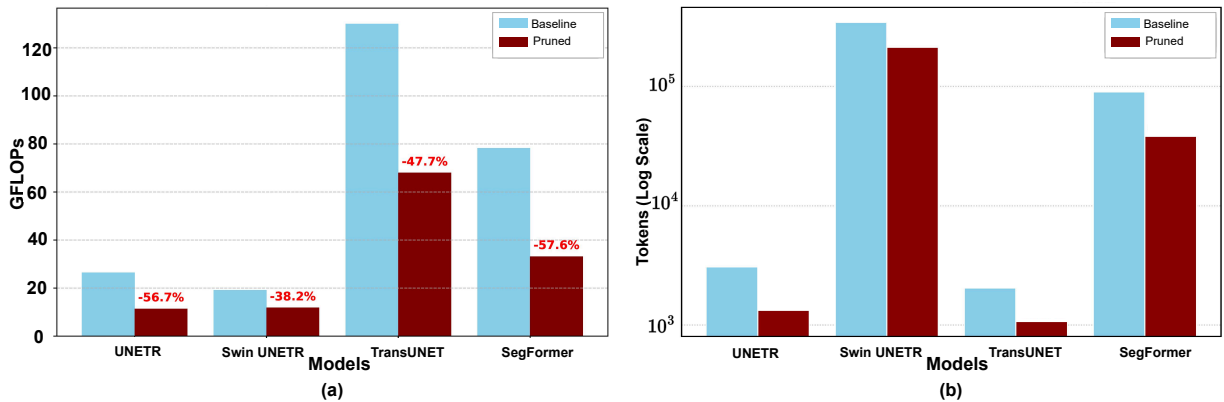}
    \caption{Graphical plot, comparing (a) GFLOPs and (b) Token Sparsity of the segmentation models with respect to their pruned versions.}
    \label{fig:barplot}
\end{figure*}
{\it PrATo} was integrated with various ViT-based state-of-the-art segmentation models to demonstrate its generalizability and model-agnostic nature. Table \ref{table:quan} summarizes the values of {\it DSC} for the segmentation models and their pruned versions using {\it PrATo} in the datasets {\it ACDC} and {\it ISIC}.  The proposed framework is applied to the output of ViT blocks for non-hierarchical models such as UNETR and TransUNET and subsequently to the patch merging stage for hierarchical models like Swin UNETR and SegFormer. This ensures that the architectural operations in the models are not compromised and the reduced set of tokens is propagated to subsequent stages. UNETR and TransUNET exhibit significant performance gains ($1\%-7\%$ approx.) in {\it DSC} values in both datasets. This suggests that the prompt-based spatial prior helps the model focus more on relevant regions. Swin UNETR and SegFormer demonstrate stable performance, even after substantial token removal, as seen in Fig. \ref{fig:barplot}. Therefore, the proposed framework improves the overall segmentation accuracy by retaining relevant tokens for target structures. 

Models with a full self-attention mechanism, such as UNETR and TransUNET, evaluate the association between all possible pairs of tokens. However, this might include processing irrelevant tokens. The proposed pruning framework functions as a regularizer, enabling the model to focus on relevant tokens corresponding to the target structures. This improves the segmentation performance. In contrast, models with local self-attention, like Swin UNETR and SegFormer, are optimized to attend the relevant local features, leading to efficient segmentation performance. Therefore, the primary benefit of {\it PrATo} incorporated variants of such models is the computational savings rather than performance improvement.

Fig. \ref{fig:barplot}(a) illustrates that the proposed framework significantly reduces computational costs across all models. The pruned variants of SegFormer and UNETR show the highest reduction in terms of GFLOPs. Fig. \ref{fig:barplot}(b) depicts a significant decrease in the token density of the pruned variants, compared to their respective baselines. This elimination of irrelevant tokens leads to computationally efficient segmentation models. Fig. \ref{fig:qual} qualitatively analyzes the effectiveness of the proposed framework. The samples highlight cases where the pruned version effectively addresses issues of over- and under-segmentation in prediction by the baselines. The experiments suggest that {\it PrATo} is a versatile framework to improve the efficacy of different segmentation models.

\begin{figure}[t]
    \centering
    \includegraphics[width=\linewidth]{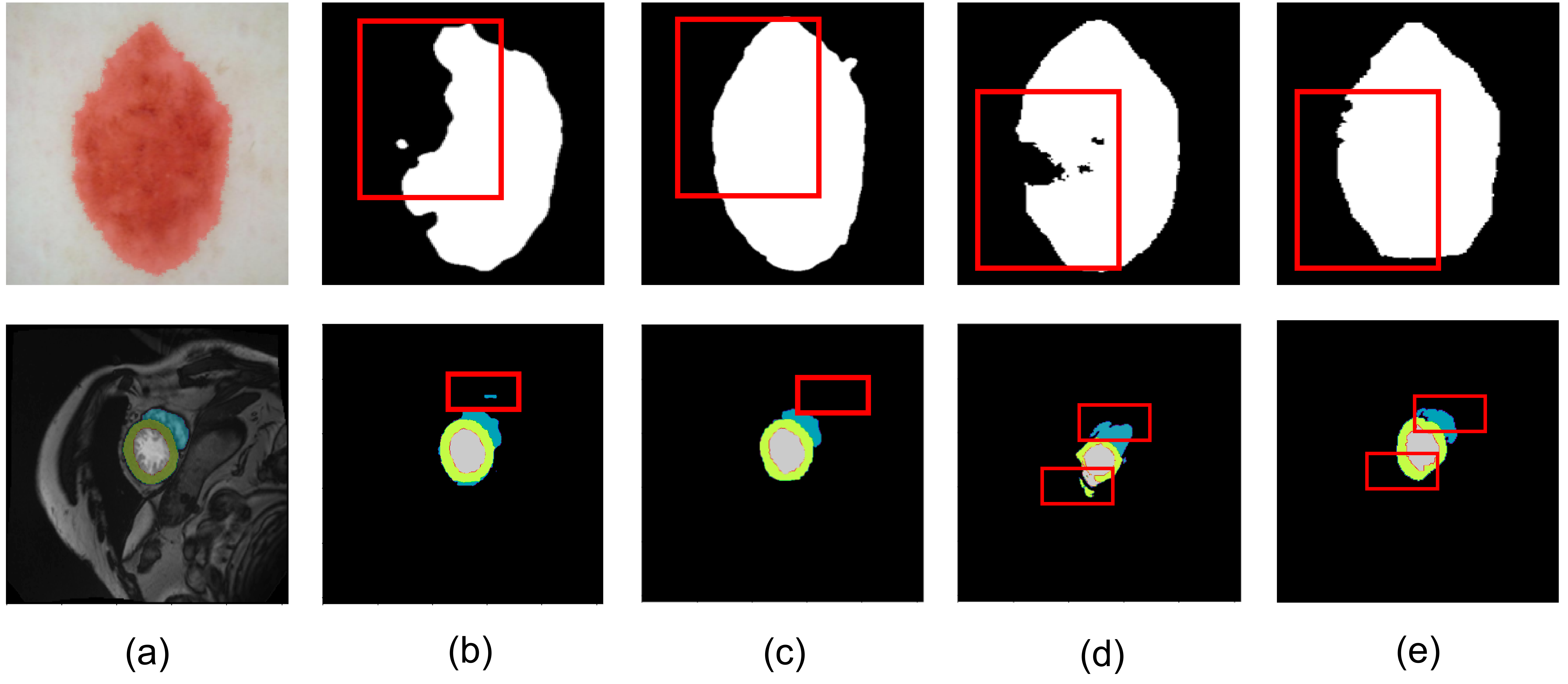}
    \caption{Qualitative results on sample images, illustrating (a) input image, with overlaid ground truth and prediction from (b) SegFormer (baseline), (c) SegFormer (pruned), (d) UNETR (baseline), and (e) UNETR (pruned).\\ Row 1: {\it ISIC}, Row 2: {\it ACDC}, sample images. Red boxes denote the comparison area.}
    \label{fig:qual}
\end{figure}
\begin{figure}[t]
    \centering
    \includegraphics[width=\linewidth]{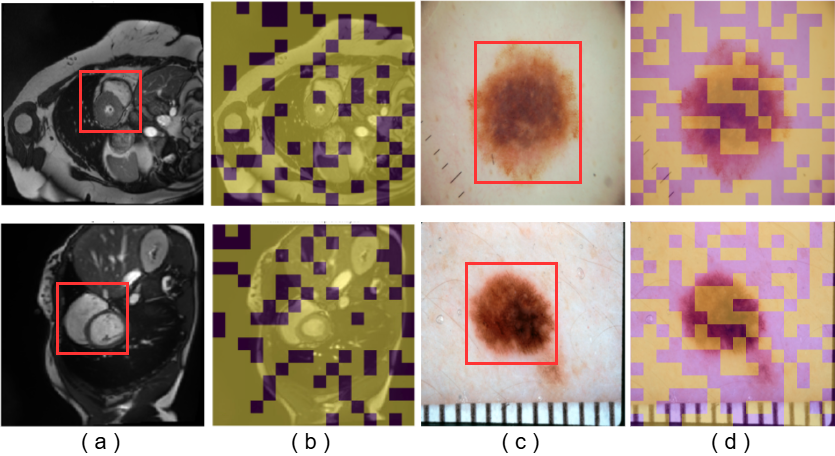}
    \caption{Sample token retention maps of {\it PrATo} from {\it ACDC} (a)-(b) and {\it ISIC} (c)-(d) datasets. The red box highlights the target region.}
    \label{fig:token_ret}
\end{figure}

\begin{figure}[t]
    \centering
    \includegraphics[width=\linewidth]{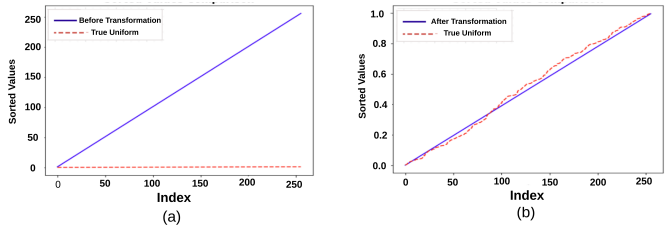}
    \caption{Line plots comparing the (a) non-uniform distribution of token relevance ranking, and (b) uniform distribution of entropy-based weighting, for an image. }
    \label{fig:com_trans}
\end{figure}

Fig. \ref{fig:token_ret} visualizes the sample token retention maps from both {\it ACDC} and {\it ISIC} datasets to gain insight into the effectiveness of the proposed approach. The retained tokens (marked by the yellow patches) correspond to the relevant spatial regions associated with the target structures. {\it PrATo} retains a dense set of tokens for the {\it ACDC} dataset, as shown in Fig. \ref{fig:token_ret}(b), which is highly concentrated within the target region. This preserves the finer details related to the different cardiac organs. In contrast, the framework intelligently adapts to the homogeneous texture of the skin lesions and prunes redundant tokens related to the uniform structure of the region of interest (Fig. \ref{fig:token_ret}(d)). This illustrates that {\it PrATo} seamlessly adapts its pruning mechanism based on the characteristics of the target structures. Fig. \ref{fig:com_trans} visualizes the effectiveness of the scoring and weighting mechanism of {\it PrATo}. Fig. \ref{fig:com_trans}(a) illustrates that the token relevance ranking distribution is non-uniform, with a small subset of highly relevant tokens. The flat `` True Uniform " line marked in red represents a vast majority of tokens have relevance scores close to zero. Therefore, these scores might lead to instability while generating the pruning mask. Fig. \ref{fig:com_trans}(b) demonstrates the weight distribution after the proposed entropy-based weighting scheme. The transformed weights follow the true uniform line, which highlights that the proposed method effectively transforms the skewed weights to a nearly uniform distribution. Therefore, such a stable weighting system indicates a robust pruning process as it prevents a few outlier tokens with extreme scores from skewing the thresholding process. 

\begin{figure}[t]
    \centering
    \includegraphics[width=\linewidth]{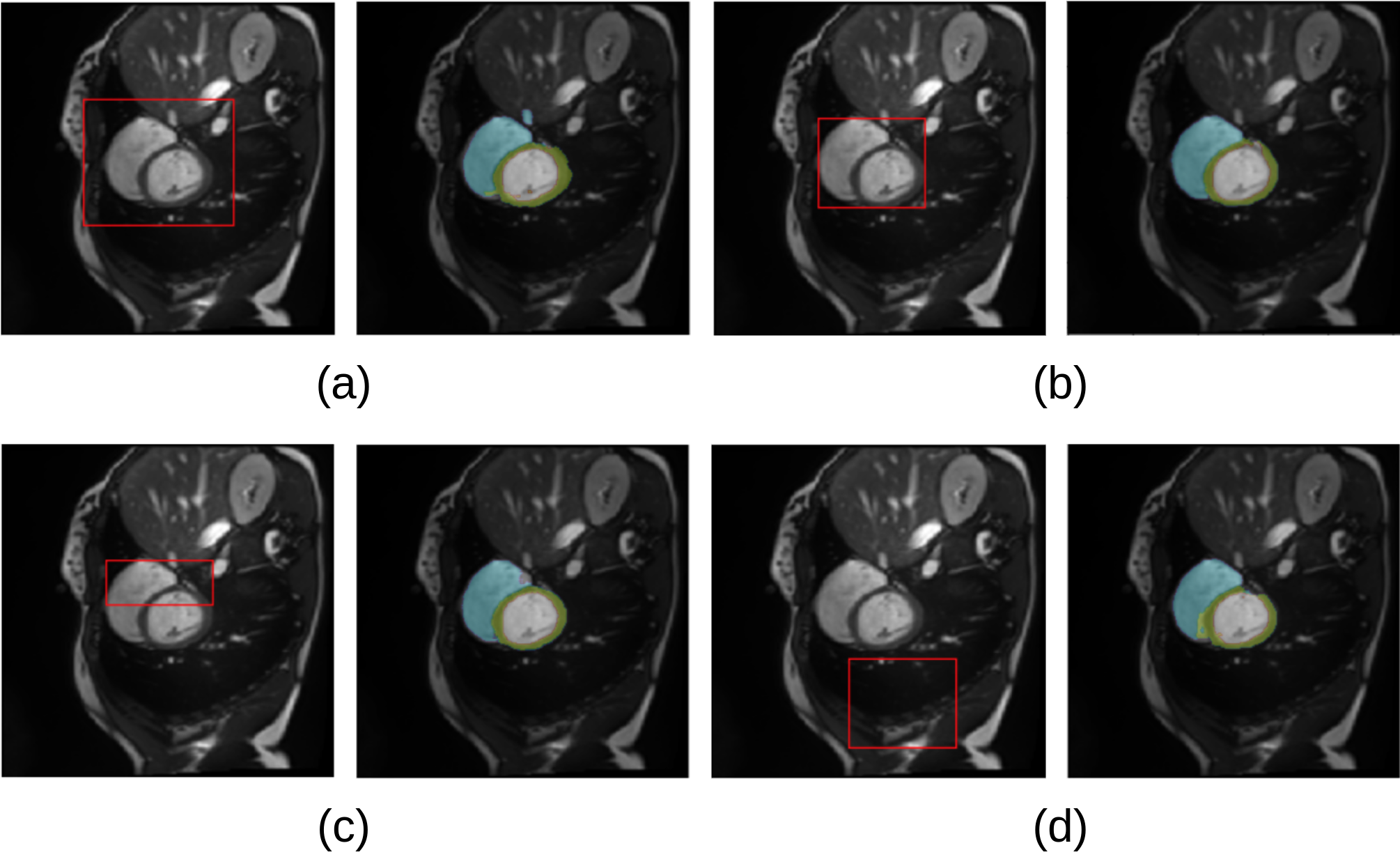}
    \caption{Qualitative results of {\it PrATo} under different prompting conditions. (a) Oversized, (b) tight-boxed, (c) partial and (d) misleading prompts.}
    \label{fig:prompt_abl}
\end{figure}

\begin{table}[t]
\centering
\label{table:prompt_abl}
\caption{Quantitative analysis of {\it PrATo} under different prompting conditions on the {\it ACDC} dataset. The best results are marked in bold.}
\resizebox{0.7\linewidth}{!}{
\begin{tabular}{|c|c|c|c|}
\hline
\textbf{Approach}           & \textbf{\textit{mDSC}}            & \textbf{\textit{mIoU}}            & \textbf{\textit{mHD95}}          \\ \hline
Tight Prompts      & \textbf{0.7169} & \textbf{0.6129} & 19.87          \\ \hline
Oversized Prompts  & 0.6571          & 0.552           & 21.43          \\ \hline
Partial Prompts    & 0.675           & 0.5763          & \textbf{17.16} \\ \hline
Misleading Prompts & 0.6528          & 0.5489          & 19.31         \\ \hline
\end{tabular}}
\end{table}

The results of the robustness analysis of the proposed framework under erroneous prompting conditions are presented in Table \ref{table:prompt_abl}. The performance of {\it PrATo} was analyzed across varying scenarios, {\it viz.}, tight, oversized, partial and misleading (incorrectly placed box) prompts. The quantitative results suggest that {\it PrATo} exhibits graceful degradation when the box prompts are placed incorrectly. The partial prompts show significant robustness with minimal drop in the mean values of $DSC$ and $IoU$. This is because partial prompts still provide high-quality features related to a portion of the target structure. The global context learned from the self-attention mechanism of ViT helps to accurately predict the remaining regions. Oversized prompts include background information along with the target region, resulting in over-segmentation, as evidenced in Fig. \ref{fig:prompt_abl}(a). The misleading prompts guide the attention of the framework away from the target region, thereby leading to the lowest average values of $DSC$ and $IoU$. However, the segmentation model learns robust feature representations, strong enough to override the misleading spatial priors. This leads to a reasonable performance. Although there are several segmentation errors(Fig. \ref{fig:prompt_abl}(d)), the overall structure is largely preserved. Thus, {\it PrATo} can handle imperfect prompts common in interactive segmentation without collapsing the overall segmentation performance of the underlying segmentation model.  

\section{Conclusion} \label{concl}

This study introduces {\it PrATo}, an adaptive framework for token pruning in Vision Transformers (ViT), designed to decrease the computational resources required to process irrelevant tokens. Using spatial priors through box prompts enhances the retention of semantically relevant tokens related to the target anatomical structure, thus improving segmentation accuracy. The calculation of the entropy-guided similarity score, to identify tokens aligned with the spatial prior, ensures that the token selection process is data-driven. This allows robustness to variations within the different input images. The experimental results suggest that the framework offers highly accurate segmentation performance and faster inference speed than other approaches. This approach reduces computational cost while maintaining (or even improving) the segmentation performance of several ViT-based segmentation models. This demonstrates the applicability of the proposed framework across different architectures. {\it PrATo} shows robustness by demonstrating graceful degradation with inaccurate prompts. Thus, the proposed framework is highly relevant for automated medical image segmentation as it tackles the high computational overhead associated with ViTs. This makes them deployable in resource-constrained settings. 

The performance of {\it PrATo} is highly dependent on the quality of the box-prompts. Furthermore, the dynamic hard pruning mechanism of {\it PrATo} may discard fine-grained features related to the boundaries of the anatomical structures. The proposed framework can be improved by integrating fine-grained prompts, such as points and scribbles, to enhance the boundary delineation capacity while preserving the efficient segmentation performance across a wide range of medical image analysis tasks and modalities.

\section{Acknowledgments}
\label{sec:acknowledgments}
This work was supported by the J. C. Bose National Fellowship, grant no. JCB/2020/000033 of S. Mitra. The authors acknowledge the use of computing resources provided by IDEAS-TIH, ISI Kolkata, India.

\bibliographystyle{ieeetr}
\bibliography{bibliography}

\begin{thebibliography}{10}

\bibitem{lecun15}
Y.~LeCun, Y.~Bengio, and {\it et al.}, ``Deep learning,'' {\em Nature}, vol.~521, pp.~436\--444, 2015.

\bibitem{dosovitskiy2020image}
A.~Dosovitskiy, ``An image is worth 16x16 words: Transformers for image recognition at scale,'' {\em arXiv preprint arXiv:2010.11929}, 2020.

\bibitem{dutta2025wavelet}
P.~Dutta, S.~Mitra, and {\it et al.}, ``Wavelet-infused convolution-transformer for efficient segmentation in medical images,'' {\em {IEEE} {T}ransactions on {S}ystems, {M}an, and {C}ybernetics: {S}ystems}, vol.~55, pp.~3326--3337, 2025.

\bibitem{yu2023marrying}
J.~Yu, T.~Ma, and {\it et al.}, ``Marrying global–local spatial context for image patches in computer-aided assessment,'' {\em {IEEE} {T}ransactions on {S}ystems, {M}an, and {C}ybernetics: {S}ystems}, vol.~53, pp.~7099--7111, 2023.

\bibitem{hatamizadeh2022unetr}
A.~Hatamizadeh, Y.~Tang, and {\it et al.}, ``{UNETR}: Transformers for 3{D} medical image segmentation,'' in {\em Proceedings of the {IEEE/CVF} Winter Conference on Applications of Computer Vision}, pp.~574--584, 2022.

\bibitem{chen2021transunet}
J.~Chen, Y.~Lu, and {\it et al.}, ``Trans{UN}et: Transformers make strong encoders for medical image segmentation,'' {\em arXiv preprint arXiv:2102.04306}, 2021.

\bibitem{hatamizadeh2021swin}
A.~Hatamizadeh, V.~Nath, and {\it et al.}, ``Swin {UNETR}: Swin transformers for semantic segmentation of brain tumors in {MRI} images,'' in {\em Proceedings of the International MICCAI {B}rain{L}esion {W}orkshop}, pp.~272--284, Springer, 2021.

\bibitem{xie2021segformer}
E.~Xie, W.~Wang, Z.~Yu, and {\it et al.}, ``Seg{F}ormer: Simple and efficient design for semantic segmentation with transformers,'' {\em Advances in {N}eural {I}nformation {P}rocessing {S}ystems}, vol.~34, pp.~12077--12090, 2021.

\bibitem{li2024a}
F.~Li, Y.~Hu, and {\it et al.}, ``A group regularization framework of {C}onvolutional {N}eural {N}etworks based on the impact of {$L_{p}$} regularizers on magnitude,'' {\em {IEEE} {T}ransactions on {S}ystems, {M}an, and {C}ybernetics: {S}ystems}, vol.~54, pp.~7434--7444, 2024.

\bibitem{yu2022width}
F.~Yu, K.~Huang, and {\it et al.}, ``Width \& depth pruning for {V}ision {T}ransformers,'' in {\em Proceedings of the {AAAI} {C}onference on {A}rtificial {I}ntelligence}, vol.~36, pp.~3143--3151, 2022.

\bibitem{lopez2024filter}
C.~I. L{\'o}pez-Gonz{\'a}lez, E.~Gasc{\'o}, and {\it et al.}, ``Filter pruning for convolutional neural networks in semantic image segmentation,'' {\em Neural Networks}, vol.~169, pp.~713--732, 2024.

\bibitem{sun2025channel}
C.~Sun, J.~Chen, and {\it et al.}, ``Channel pruning method driven by similarity of feature extraction capability,'' {\em Soft Computing}, pp.~1--20, 2025.

\bibitem{adnan2024structured}
M.~Adnan, Q.~Ba, and {\it et al.}, ``Structured model pruning for efficient inference in computational pathology,'' in {\em International workshop on {M}edical {O}ptical imaging and {V}irtual microscopy {I}mage analysis {(MOVI)}}, pp.~140--149, Springer, 2024.

\bibitem{marchetti2025efficient}
M.~Marchetti, D.~Traini, and {\it et al.}, ``Efficient token pruning in {V}ision {T}ransformers using an attention-based multilayer network,'' {\em Expert {S}ystems with {A}pplications}, vol.~279, p.~127449, 2025.

\bibitem{zhang2025intra}
P.~Zhang, C.~Tian, and {\it et al.}, ``Intra-head pruning for {V}ision {T}ransformers via inter-layer dimension relationship modeling,'' {\em Neural {N}etworks}, p.~107656, 2025.

\bibitem{ren2022shunted}
S.~Ren, D.~Zhou, and {\it et al.}, ``Shunted self-attention via multi-scale token aggregation,'' in {\em Proceedings of the {IEEE/CVF} {C}onference on {C}omputer {V}ision and {P}attern {R}ecognition ({CVPR})}, pp.~10853--10862, 2022.

\bibitem{he2023structured}
Y.~He and L.~Xiao, ``Structured pruning for deep convolutional neural networks: {A} survey,'' {\em {IEEE} {T}ransactions on {P}attern {A}nalysis and {M}achine {I}ntelligence}, vol.~46, pp.~2900--2919, 2023.

\bibitem{liu2024efficient}
S.~Liu, L.~Wang, and W.~Yue, ``An efficient medical image classification network based on multi-branch {CNN}, token grouping transformer and mixer {MLP},'' {\em Applied Soft Computing}, vol.~153, p.~111323, 2024.

\bibitem{rao2021dynamicvit}
Y.~Rao, W.~Zhao, and {\it et al.}, ``Dynamic{V}i{T}: Efficient {V}ision {T}ransformers with {D}ynamic {T}oken {S}parsification,'' {\em Advances in {N}eural {I}nformation {P}rocessing {S}ystems}, vol.~34, pp.~13937--13949, 2021.

\bibitem{lin2023the}
X.~Lin, L.~Yu, and {\it et al.}, ``The lighter the better: {R}ethinking transformers in medical image segmentation through adaptive pruning,'' {\em {IEEE} {T}ransactions on {M}edical {I}maging}, vol.~42, pp.~2325--2337, 2023.

\bibitem{salam2025skin}
A.~A. Salam, M.~Z. Asaf, and {\it et al.}, ``Skin whole slide image segmentation using lightweight-pruned transformer,'' {\em Biomedical {S}ignal {P}rocessing and {C}ontrol}, vol.~106, p.~107624, 2025.

\bibitem{xu2022evo}
Y.~Xu, Z.~Zhang, and {\it et al.}, ``Evo-{V}i{T}: {S}low-fast token evolution for dynamic {V}ision {T}ransformer,'' in {\em Proceedings of the {AAAI} {C}onference on {A}rtificial {I}ntelligence}, vol.~36, pp.~2964--2972, 2022.

\bibitem{chen2024ta}
J.~Chen, X.~Zhang, and {\it et al.}, ``{TA-ASF}: {A}ttention-sensitive token sampling and fusing for visual {T}ransformer models on the edge,'' in {\em Proceedings of the {IEEE/ACM} {S}ymposium on {E}dge {C}omputing ({SEC})}, pp.~123--134, IEEE, 2024.

\bibitem{zhou2023token}
L.~Zhou, H.~Liu, and {\it et al.}, ``Token sparsification for faster medical image segmentation,'' in {\em Proceedings of the {I}nternational {C}onference on {I}nformation {P}rocessing in {M}edical {I}maging ({IPMI})}, pp.~743--754, Springer, 2023.

\bibitem{tang2023dynamic}
Q.~Tang, B.~Zhang, and {\it et al.}, ``Dynamic token pruning in plain vision transformers for semantic segmentation,'' in {\em Proceedings of the {IEEE/CVF} {I}nternational {C}onference on {C}omputer {V}ision}, pp.~777--786, 2023.

\bibitem{tian2025beyond}
Y.~Tian, L.~Xie, and {\it et al.}, ``Beyond masking: Demystifying token-based pre-training for {V}ision {T}ransformers,'' {\em Pattern Recognition}, vol.~162, p.~111386, 2025.

\bibitem{huang2024general}
W.~Huang, Y.~Shen, and {\it et al.}, ``A general and efficient training for {T}ransformer via token expansion,'' in {\em Proceedings of the {IEEE/CVF} {C}onference on {C}omputer {V}ision and {P}attern {R}ecognition ({CVPR})}, pp.~15783--15792, 2024.

\bibitem{bernard2018deep}
O.~Bernard, A.~Lalande, and {\it et al.}, ``Deep learning techniques for automatic {MRI} cardiac multi-structures segmentation and diagnosis: {I}s the problem solved?,'' {\em {IEEE} {T}ransactions on {M}edical {I}maging}, vol.~37, pp.~2514--2525, 2018.

\bibitem{codella2018skin}
N.~C. Codella, D.~Gutman, and {\it et al.}, ``Skin lesion analysis toward melanoma detection,'' in {\em Proceedings of the {IEEE} 15th {I}nternational {S}ymposium on {B}iomedical {I}maging ({ISBI})}, pp.~168--172, IEEE, 2018.

\bibitem{ronneberger2015u}
O.~Ronneberger, P.~Fischer, and {\it et al.}, ``{\it U}-{N}et: {C}onvolutional networks for biomedical image segmentation,'' in {\em Proccedings of the {I}nternational {C}onference on {M}edical {I}mage {C}omputing and {C}omputer {A}ssisted {I}ntervention, ({MICCAI})}, pp.~234\--241, Springer, 2015.

\bibitem{lei2016layer}
J.~Lei~Ba, J.~R. Kiros, and {\it et al.}, ``Layer normalization,'' {\em Ar{X}iv e-prints}, pp.~ar{X}iv--1607, 2016.

\bibitem{jia2022visual}
M.~Jia, L.~Tang, and {\it et al.}, ``Visual prompt tuning,'' in {\em Proceedings of the {E}uropean {C}onference on {C}omputer {V}ision}, pp.~709--727, Springer, 2022.

\bibitem{he2017mask}
K.~He, G.~Gkioxari, and {\it et al.}, ``Mask {R}-{CNN},'' in {\em Proceedings of the {IEEE} {I}nternational {C}onference on {C}omputer {V}ision ({ICCV})}, pp.~2961--2969, 2017.

\bibitem{shannon1948mathematical}
C.~E. Shannon, ``A mathematical theory of communication,'' {\em The {B}ell {S}ystem {T}echnical {J}ournal}, vol.~27, pp.~379--423, 1948.

\bibitem{taghanaki2019combo}
S.~A. Taghanaki, Y.~Zheng, and {\it et al.}, ``Combo loss: {H}andling input and output imbalance in multi-organ segmentation,'' {\em Computerized {M}edical {I}maging and {G}raphics}, vol.~75, pp.~24--33, 2019.

\bibitem{nguyen2023manet}
T.-C. Nguyen, T.-P. Nguyen, and {\it et al.}, ``{MAN}et: {M}ulti-branch attention auxiliary learning for lung nodule detection and segmentation,'' {\em Computer {M}ethods and {P}rograms in {B}iomedicine}, vol.~241, p.~107748, 2023.

\end{thebibliography}

\end{document}